\documentclass[12pt]{article}
\usepackage{times}
\usepackage{graphicx}
\usepackage{color}
\usepackage{multirow}
\usepackage[authoryear]{natbib}
\usepackage{rotating}
\usepackage{bbm}
\usepackage{latexsym}

\usepackage{booktabs}
\usepackage{algorithm}
\usepackage{algorithmic}
\usepackage{amssymb}
\usepackage{amsmath}
\usepackage{amsthm}
\usepackage{setspace}
\usepackage{url}
\usepackage{secdot}
\usepackage{xcolor}

\newtheorem{thm}{Theorem}

\newcommand{\captionfonts}{\normalsize}

\makeatletter  
\long\def\@makecaption#1#2{%
  \vskip\abovecaptionskip
  \sbox\@tempboxa{{\captionfonts #1: #2}}%
  \ifdim \wd\@tempboxa >\hsize
    {\captionfonts #1: #2\par}
  \else
    \hbox to\hsize{\hfil\box\@tempboxa\hfil}%
  \fi
  \vskip\belowcaptionskip}
\makeatother   

\newcommand\todo[1]{{#1}}

\begin{document}
\hspace{13.9cm}1

\ \vspace{20mm}\\

{\LARGE Using \todo{inspiration from} synaptic plasticity rules to optimize traffic flow in distributed engineered networks}

\ \\
{\bf \large Jonathan Y.\@ Suen$^{\displaystyle 1}$ and Saket Navlakha$^{\displaystyle 2}$}\\
{$^{\displaystyle 1}$Duke University, Department of Electrical and Computer Engineering.\\  Durham, North Carolina, 27708 USA}\\
{$^{\displaystyle 2}$The Salk Institute for Biological Studies. Integrative Biology Laboratory.\\10010 North Torrey Pines Rd., La Jolla, CA 92037 USA}\\
%


\thispagestyle{empty}
\markboth{}{NC instructions}
\ \vspace{-0mm}\\
%
\begin{center} {\bf Abstract} \end{center}

Controlling the flow and routing of data is a fundamental problem in many distributed networks, including transportation systems, integrated circuits, and the Internet. In the brain, synaptic plasticity rules have been discovered that regulate network activity in response to environmental inputs, which enable circuits to be stable yet flexible. Here, we develop a new neuro-inspired model for network flow control that only depends on modifying edge weights in an activity-dependent manner. We show how two fundamental plasticity rules (long-term potentiation and long-term depression) can be cast as a distributed gradient descent algorithm for regulating traffic flow in engineered networks. We then characterize, both via simulation and analytically, how different forms of edge-weight update rules affect network routing efficiency and robustness. We find a close correspondence between certain classes of synaptic weight update rules derived experimentally in the brain and rules commonly used in engineering, suggesting common principles to both.

\section{Introduction}

 
In many engineered networks, a payload needs to be transported between nodes without central control. These systems are often represented as weighted, directed graphs. Each edge has a fixed \emph{capacity}, which represents the maximum amount of traffic the edge can carry at any time. Traffic in these networks consist of a series of \emph{flows}, each of which contains some amount of data that originates at a source node and attempts to reach a target node via a path in the network. For example, in vehicular transportation systems (nodes are intersections, edges are roads), cars travel from one location to another, and each road has a capacity that limits the number of cars that can traverse the road at once. In network-on-a-chip circuits (nodes are components such as CPU and GPU cores equipped with extremely basic routers, and edges are circuit wiring between cores), tiny units of data called flits flow through the circuit, and each link can only transport a limited number of flits at a time~\citep{Cota2012}. On the Internet, packets navigate from one host to another, and each communication link has a capacity that constrains the rate of data flow through the link. In all these cases, the goal is to optimize performance metrics, including how long it takes for data to reach the target and how long data is queued or lost along the way due to exceeding capacities.


\todo{There are two primary services required in these networks: routing and flow control. Network routing refers to moving a payload from a source node in the network to a target node through some path~\citep{Royer1999}, often maintained using a distributed routing table~\citep{Gavoille2001}. The majority of networks always send traffic along the shortest path from source to target, where shortest refers to physical distance or transit time. In this paper, we also assume routing occurs via the shortest source-target path and focus on algorithms for flow control.}


\todo{Flow control determines when and how much data can be sent through the network.} Engineered networks are often designed based on the concept of oversubscription or blocking, where there is not enough capacity to simultaneously service every possible flow at maximum capacity due to bottlenecks in the network. Bottlenecks can cause \emph{congestion}, which results in data loss and reduction in useful throughput since when demand for a certain link exceeds capacity, the excess data must be queued or dropped. \todo{Effective flow control simultaneously maximizes the utilization of link capacity while maintaining low loss and delay. In general, flow control is an NP-hard problem~\citep{Wang2003,Fortz2000}. It is even more challenging in online systems that serve many flows concurrently, where traffic can change unpredictably, and where optimization must happen in real-time. Further, control logic must be implemented distributedly, with minimal communication between nodes.}


To address these challenges, we developed a new neuro-inspired distributed computing model where flow control occurs by modulating (increasing or decreasing) edge \emph{weights} using only 1-bit \todo{(binary)} local feedback. Edge weights represent a control variable that denotes how much data should flow along the edge at the current time. For example, if there are 10 units of data that want to travel from node $u$ to $v$, and if the weight of edge $(u,v)$ is 6, then 6 units can be successfully transmitted, and the other 4 units are either queued or dropped. \todo{There exists an optimal global weight distribution at every time step, which is dependent on the trade-off between maximizing data flow rates versus minimizing data drops and queueing.} To inform the direction of how weights should change to approach this distribution, the network relies on 1-bit local feedback between neighboring nodes, which indicates whether the data was successfully transmitted (without being queued or dropped). This feedback is used to increase or decrease edge weights to ensure that links are not under-utilized or overloaded, respectively.


How does this relate to synaptic plasticity in the brain? Building off prior work, we will argue that synaptic weight-update rules can be viewed as a distributed gradient descent process that attempts to find a weight distribution that optimizes a global objective~\citep{Bengio2015b,Bengio2015a}. While the computational models and feedback mechanisms used to trigger weight changes in engineered networks are clearly different than those used in the brain, the questions we consider here are: 1) The direction and magnitude of weight updates; i.e.\@ when and how much to increase and decrease; 2) How these local decisions affect global performance objectives in engineering (bandwidth, drops, queueing delay); and 3) Whether general principles for modulating edge weights may be found across engineered and neural systems.


Overall, this paper makes four contributions: 1) A new neuro-inspired distributed computing model to optimize traffic flow based only on edge weight updates and 1-bit feedback between adjacent nodes; 2) A casting of long-term potentiation (LTP) and long-term depression (LTD) in terms of distributed gradient descent dynamics; 3) Simulations, using simulated and real networks, and theoretical analysis of five classes of weight-update rules; and 4) Comparisons of the best performing classes with experimental data detailing the functional forms of LTP and LTD in the brain.

\subsection{Related work}

Congestion control in many distributed networks, such as the Internet, is performed on a global end-to-end basis \emph{per-flow}, meaning that the source of each flow regulates the rate at which it sends data into the network based on a binary feedback signal from the target node~\citep{Corless2016}. Our neuro-inspired \emph{per-link} model described below is stricter in its feedback constraints; each node can regulate traffic but based only on the congestion it observes on its incoming and outgoing links, independently of the source and target of the data. This model is more relevant to vehicular traffic networks (where it is impossible for a traffic control device, such as a traffic light, to know the ultimate destination of a vehicle) and network-on-a-chip circuits (where flits travel independently through the circuit, like vehicles). Traffic control algorithms have been analyzed in many forms, but largely assume the problem is centralized or offline~\citep{Gayme2011}. \todo{Other approaches attempt to emulate a centralized algorithm by passing large-sized messages~\citep{MateosNunez2013}}. Distributed gradient descent algorithms have been studied in many areas, but they also assume the ability to pass large and frequent messages, or they require significant data aggregation at individual nodes~\citep{Li2014,Zinkevich2010,Yuan2014}.

\section{A distributed network model for traffic flow control}

We are given a network $G = (V,E,W)$. The node set $V$ is partitioned into three non-overlapping subsets: sources $S \subset V$, targets $T \subset V$, and routers $R \subset V$. Data is transmitted from sources to target via the routers (Figure~1A). The edges $E$ are directed; for simplicity, we assume each source is connected to exactly one router (i.e.\@ a ``feeder road''), and each target has an incoming edge from exactly one router (i.e.\@ a ``highway exit''). The routers are connected with a uniform or scale-free degree topology. 


The weight $W_{uv}(t)$ of each edge corresponds to the maximum amount of traffic \todo{the flow control algorithm allows to travel} from $u\rightarrow v$ at time $t$. Each weight $W_{uv}(t) \in [1,C]$ for all $(u,v) \in E$ and for all $t$, where $C$ is defined as a fixed maximum capacity for each edge. Edge weights from routers to targets are always set to $C$ because there are no outgoing edges from the targets, and hence no possibility of downstream congestion. Weights of all other other edges will vary over time based on traffic flow and congestion.  

Each source desires to send $L \gg C$ units of data to one random target; hence, there are $|S|$ flows actively competing in the network and sending data. Data for each flow is routed from the source to the target via the shortest path in the network. In practice, this is easily accomplished using a distributed routing algorithm~\citep{Gavoille2001}\todo{\footnote{\todo{This model of routing is likely different than that occurring in the brain. Our work is not meant to derive a mapping between the two or suggest that their mechanisms are similar.}}} In each time step $t$, each source injects new data into the network, the amount of which is equal to the outgoing edge weight from the source to the router it is connected to. Thus, if this weight changes, so does the amount of new data injected into the network by the source. Data arriving at a router is forwarded to the next node along the shortest path to the target. If data arrives at router $u$ destined for node $v$, the amount of data actually sent is upper-bounded by the edge weight $W_{uv}(t)$, which can change over time. \todo{If multiple flows desire to use edge $u\rightarrow v$ in the same time step, we process each flow in a random order. This means that a random incoming flow is chosen, its data units are all sent (up to the link weight; the rest are dropped). Then another incoming flow is chosen at random, and its units are handled similarly.} A flow is complete when its source has successfully transferred $L$ data units to its target. 

\todo{The only control variables of the algorithm to achieve flow control are the edge weights $W$ at each time step.} Our objective is to:
\begin{align}
  \textrm{maximize}&\quad \sum_t \mathcal{U}_\mathcal{F}(W(t)) \\
  \textrm{subject to}&\notag\\
  &W_{uv}(t) \leq C \qquad\text{for all $t$ and $(u,v) \in E$}\notag
%
%
\end{align}
Here, $\mathcal{U}_\mathcal{F}$ is an objective function (described later) that measures how well the current edge weights perform when routing data from the flows $\mathcal{F}$ currently in the network. To enforce the constraint, we propose two application-dependent models for penalizing excess traffic on an edge: a \emph{drop} model and a \emph{queue} model (Figure~1B). Let $D_{uv}(t)$ be the total amount of data at router $u$ that desires to be passed to node $v$ at time $t$. If $D_{uv}(t) > W_{uv(t)}$ then the excess data is either dropped (i.e.\@ discarded), or it is queued at node $u$ until at least the next time step. Loss models are consistent with many data networks and neural circuits~\citep{Branco2009}. Queue models are consistent with transport networks and can help smoothen transient or bursty traffic. In each time step, data queued at a node is processed before non-queued data (i.e.\@ a first-in first-out buffer). We assume a queue of infinite length, but we penalize solutions that produce long queues.

Thus, at every time step, there exists an optimal global weight distribution for a given objective function. \todo{This optimum will change as traffic demand varies over time. The goal is to track this distribution as closely as possible by applying gradient descent on the edge weights. We use 1-bit feedback between adjacent nodes, indicating whether data was successfully sent from one node to the other (i.e.\@ if it resulted in no dropped or queued data). If the transmission is successful, data rate is increased (to probe whether higher bandwidth can be achieved). Conversely, if congestion is experienced, the data rate is decreased. This feedback thus serves as the direction of the local gradient (at the edge) for the global objective. By modifying the edge weight in accordance with this local gradient, we attempt to minimize the global objective. The magnitude of the edge weight change is based on both the direction of change and the current weight. Therefore, the algorithm seeks the point that locally maximizes traffic flow without triggering congestion.}

Next, we describe how synaptic plasticity rules can serve as inspiration towards regulating these edge weights in an online, activity-dependent manner using simple distributed computation.

\subsection{Synaptic plasticity as distributed gradient descent on edge weights}

Recent work by Bengio et al.\@ has argued that many forms of Hebbian learning, such as spike-timing dependent plasticity, may correspond to gradient descent towards neural activities that optimize a global objective~\citep{Bengio2015b,Bengio2015a,Osogami2015}. They propose that neurons perform ``inference'' to try and better predict future activity, given current and past data. To approach optimal activity levels, feedback signals between pre-and post-synaptic neurons, such as those used to trigger long-term potentiation (LTP) and long-term depression (LTD), cause firing rates to increase or decrease based on the gradient of the objective.

\todo{Under a connectionist assumption, where} neural activity is a function of the synaptic weights coupling neurons together, the state of the network can be described by the edge weights over the population of synapses (along with other presumed constants, such as activation functions). The evolution of the system can be described by how these weights change in response to activity. In our case, activity-dependent feedback signals between adjacent nodes (described in detail below) provides a measurement of the direction of the local gradient at the edge, which similarly triggers edge weights to increase or decrease. These weight changes thus correspond to a distributed gradient descent algorithm for finding a set of edge weights that maximizes a global objective. The movement towards the global optimal is complicated by the non-independence of weight changes (network effects) and the uncertainty in future inputs (non-stationarity of traffic).

One critical question then is: How much should the weight increase (``LTP'') or decrease (``LTD'') following feedback? We next describe experimental data detailing what forms these rules might take in the brain.

\subsection{Experimentally-derived weight update rules for LTP and LTD}

To inspire our search into different possible weight update rules, we surveyed the recent literature for \todo{models based on} experimental data (electrophysiology, imaging, etc.) that provided evidence of the functional forms of LTP and LTD (Table~\ref{tbl:rw}). We categorized rules into four classes: 1) Additive: the change in edge weight is based on an additive constant; 2) Multiplicative: the change is based on a multiplicative constant; 3) Weight-dependent: the change more generally is based on a function of the existing edge weight; and 4) Time-dependent: the change also depends on the history of recent edge-weight changes. In this paper, we focus on the first three classes. 

These rules will be used to derive a class of neuro-inspired distributed gradient descent algorithms for update edge weights, as described in the next section. Table~\ref{tbl:rw} is not meant to exhaustively list all forms of synaptic plasticity rules derived in the literature (see Discussion) but rather to provide some basic structure into possible, simple-to-implement rules and their parameters. 

\begin{sidewaystable}[h]
\begin{center}
\begin{tiny}
\caption{\textbf{Correspondence between synaptic plasticity rules with experimental support and network control algorithms used in engineering.} The ``---'' indicates lack of support. Reviews of many of these rules are provided by: \citet{Kepecs2002,Morrison2008,Watt2010,Froemke2010}. [1]: \citep{Song2000,Kopec2006}. [2]: \citep{Chiu1989,Corless2016}. \todo{History-dependent rules may also depend on the history of recent activity, instead of only the history of recent weights.}}
\begin{tabular}{lllll} \\ \toprule 
Plasticity & Rule & Equation & Brain & Engineering \\ \midrule
LTP & Additive              & $W(t+1) = W(t) + k_i$ & \citep{Song2000,Kopec2006} & \citep{Chiu1989,Corless2016} \\ 
    & Multiplicative  & $W(t+1) = W(t) \times k_i$ & --- & \citep{Chiu1989} \\ 
    & Weight-Dependent & $W(t+1) = W(t) \mathop{}^{+}_{\times} f(W(t))$ & \citep{Oja1982,Bi1998,vanRossum2000} & \citep{Bansal2001,Kelly2003,Ha2008}\\
    & History-Dependent     & $W(t+1) = f(W(t),\dots,W(t-\Delta(t)))$ & \citep{Bienenstock1982,Froemke2010,Cooper2012,Kramar2012} & \citep{Jin2001} \\ \midrule 
LTD & Subtractive           & $W(t+1) = W(t) - k_d$ & \citep{Song2000} & --- \\ 
    & Multiplicative  & $W(t+1) = W(t) \times k_d$ & \citep{Bi1998,vanRossum2000,Zhou2004} & \citep{Chiu1989,Corless2016} \\ 
    & Weight-Dependent & $W(t+1) = W(t) \mathop{}^{-}_{\times} f(W(t))$ & \citep{Oja1982} & \citep{Bansal2001,Corless2012} \\
    & History-Dependent     & $W(t+1) = f(W(t),\dots,W(t-\Delta(t)))$ & \citep{Bienenstock1982,Froemke2010,Cooper2012} & \citep{Jin2001} \\ 
%
\bottomrule
\end{tabular}
\begin{flushleft}
\end{flushleft}
\label{tbl:rw}
\end{tiny}
\end{center}
\end{sidewaystable}

\subsection{Distributed algorithms for updating edge weights}
\label{sec:alg}




First, to inform the direction of the weight change (increase or decrease), in each time step, we allow 1-bit \emph{feedback} between adjacent nodes. Let $\mathrm{Jam}_{vw}(t)$ be an indicator variable equal to $1$ if data at node $v$ is lost or queued due to congestion on edge $(v,w)$ at time $t$, and $0$ otherwise. Assume data for a flow is traveling from node $u$ to $v$ to $w$. \todo{When considering how to modify the weight of edge $(u,v)$, we need to consider what happens on the adjacent downstream edge $(v,w)$ where traffic is flowing. Intuitively, if $\mathrm{Jam}_{vw}(t)==1$, then the incoming edge weight $W_{uv}(t+1)$ should LTD since it contributed data to the jam. Further, the edge weight $W_{vw}(t+1)$ should LTP to attempt to alleviate the congestion.} If neither edge $(u,v)$ nor $(v,w)$ are jammed, then both should LTP (Figure~1C). Thus, the Jam term serves as a 1-bit measurement of the direction of the local gradient at edge $(u,v)$. Overall, the logic implemented at each edge $(u,v)$ is shown in Algorithm~1 below:

\begin{algorithm}
\caption{: Logic for applying LTP or LTD, implemented at each edge $(u,v)$}
\begin{algorithmic}
  \IF{$\mathrm{Jam}_{uv}(t)==1$}
    \STATE Apply LTP to $W_{uv}(t+1)$ 
  \ELSE
    \IF{$\mathrm{Jam}_{vw}(t)==1$ and $(u,v)$ contributed data that jammed $(v,w)$, for any $w$}
        \STATE Apply LTD to $W_{uv}(t+1)$ 
    \ELSIF{$\mathrm{Jam}_{vw}(t)==0$ for all $w$ that received data from $u$}
        \STATE Apply LTP to $W_{uv}(t+1)$ 
    \ENDIF
  \ENDIF
\end{algorithmic}
\end{algorithm}
%
%
%
We assume a node (in this case, $v$) can modify the edge weight of both its incoming and outgoing edges. In synapses, this may be achieved by modulating pre-synaptic release probability or number of post-synaptic receptors \todo{(e.g.\@~\citet{Costa2015,Markram2012,Yang2013,Fitzsimonds1997}) or other gating mechanisms~\citep{Vogels2009}}; in data networks, a node can pause incoming data by transmitting a signal and can pause outgoing data by simply not transmitting. If an edge gets both LTP and LTD signals in a time step, it default to LTD. In every time step, the weight of every edge that carries data will either LTP or LTD. 




Second, to determine the magnitude of the weight change, we consider the following weight-update rules for LTP and LTD:
\begin{eqnarray}
\label{eqn:ltp}
W_{uv}(t+1) & = &
  \begin{cases}
     W_{uv}(t) + k_i,\ k_i > 0 & \quad\textbf{LTP: Additive Increase (AI)}\\
     W_{uv}(t) \times k_i,\ k_i > 1 & \quad\textbf{LTP: Multiplicative Increase (MI)}\\
     W_{uv}(t) - k_d,\ k_d > 0 & \quad\textbf{LTD: Subtract Decrease (SD)}\\
     W_{uv}(t) \times k_d,\ 0 < k_d < 1 & \quad\textbf{LTD: Multiplicative Decrease (MD)}\\
  \end{cases} \nonumber
\end{eqnarray}

We consider four combinations of LTP and LTD: AIMD, AISD, MIMD, and MISD. Each of these algorithms has some theoretical or experimental basis, or both (Table~\ref{tbl:rw}). For example, AISD was proposed by~\citet{Kopec2006} and~\citet{Song2000}. AIMD was proposed by~\citet{vanRossum2000} and~\citet{Delgado2010}. Multiplicative decrease rules have been proposed by~\citet{Zhou2004}, amongst others. 

\todo{}

We also compare to an algorithm based on the classic Oja learning rule~\citep{Oja1982}:
\begin{eqnarray}
\label{eqn:oja}
W_{uv}(t+1) & = &
  \begin{cases}
     W_{uv}(t) + k_i\left(1-\frac{D_{uv}(t)^2}{W_{uv}(t)C}\right),\ k_i > 0 & \quad\textbf{LTP: Oja}\\
     W_{uv}(t) - k_d\left(1+\frac{D_{uv}(t)^2}{W_{uv}(t)C}\right),\ k_d > 0 & \quad\textbf{LTD: Oja}\\
  \end{cases} \nonumber
\end{eqnarray}
Unlike the previous rules, Oja uses the activity (traffic) of the edge as a variable, where $D_{uv}(t)$ is the amount of data traversing edge $(u,v)$ at time $t$. \todo{This rule is slightly different from the typical Oja rule where the change in the weight for the $i$\textsuperscript{th} input, $\Delta W_i=\alpha(x_iy-y^2W_i)$ (learning weight $\alpha$, synaptic input $x$ and output $y$). The squared term functions as a decay on the weight. We include a similar activity-dependent squared decay term $(D_{uv}(t)^2/W_{uv}(t))$, but we normalize it by $C$ to lie within the required weight range. Our term decreases the effect of LTP and increases the effect of LTD as the link approaches capacity. This allows more aggressive $k_i$ and $k_d$ to quickly adjust traffic.}

Another algorithm we compare is called Bang-Bang control. This rule is often used in neural circuit design~\citep{Feng2003,Zanutto2007} and in engineering~\citep{Lazar1983} to control and stabilize activity:
\begin{eqnarray}
\label{eqn:other}
W_{uv}(t+1) & = &
  \begin{cases}
     C & \quad\textbf{LTP: Bang-Bang}\\
     1 & \quad\textbf{LTD: Bang-Bang}\\
  \end{cases} \nonumber
\end{eqnarray}
Finally, we compare to a baseline rule, Max Send, which keeps all edge weights fixed at $C$ in every time step. For all algorithms, if a weight equals $C$ and is triggered to LTP, it stays at $C$. Likewise, if a weight equals $1$ and is triggered to LTD, it stays at $1$.

\todo{We only consider integer units of data, thus for all update rules, $W$ is rounded to an integer. To prevent a link from getting stuck at 0 weight (e.g.\@ for MISD), it was required that every weight have a minimum of 1. The link capacity $C \gg 1$, so the integer rounding and minimum value were negligible in terms of overall performance.}

\subsection{Objective functions, simulations, and data}
\label{sec:eval}

Next, we describe network performance measures to quantify how well these rules behave towards optimizing global objectives ($\mathcal{U}_\mathcal{F}$). The objective functions we selected are typically used to evaluate performance in engineered networks\todo{~\citep{Pande2005,Ahn1995}}. 

Let $\mathcal{F}$ define the set of $|S|$ competing flows, and $L$ be the load of each flow. The three objective functions are:
\begin{itemize} 
\item \underline{Bandwidth}: Amount of data successfully transferred per time step, averaged over all flows: $|\mathcal{F}|^{-1} \sum_i L / \mathrm{time}(\mathcal{F}_i)$, where $\mathrm{time}(\mathcal{F}_i)$ is the number of time steps for flow $i$ to complete. 
\item \underline{Drop Penalty}: Percentage of data lost, averaged over all flows: $|\mathcal{F}|^{-1} \sum_i \mathrm{lost}(\mathcal{F}_i) / L$, where $\mathrm{lost}(\mathcal{F}_i)$ is the amount of data lost by flow $i$ over all time steps until completion. The drop penalty may go above 100\% if more data sent by the source is lost than delivered.
\item \underline{Queue Penalty}: The percentage of data that is queued per hop, averaged over all flows: $|\mathcal{F}|^{-1} \sum_i \mathrm{queued}(\mathcal{F}_i) / (L \times \mathrm{Path}(\mathcal{F}_i))$, where $\mathrm{queued}(\mathcal{F}_i)$ is the total amount of data inserted into queues and $\mathrm{Path}(\mathcal{F}_i)$ is the path length of flow $i$.
\item \todo{\underline{Parameter robustness}: One critical component of these algorithms is that they must work well in general. Traffic is highly dynamic, and thus optimizing parameters for one particular traffic regime or network topology will not be sufficient for real-world use. We thus focus on the robustness or sensitivity of each algorithm by testing the variability in their performance across a broad range of parameters.}
\end{itemize} 

\noindent \textbf{Simulation framework.} We created a directed network with $N$ sources, $N$ targets, and $N$ routers, where $N=100$ or $1000$. \todo{Each source is connected to exactly one random router (the same router can have an edge from multiple different sources)}, and each target has an incoming edge from one random router. The router-router network was defined using a uniform or scale-free degree topology with each router connected to six other routers~\citep{Corless2016}. We defined $N$ concurrent flows, one starting from each source, and each selecting a random target (the same target may be selected twice across flows). Each flow contained $L = 100 \times C$ data units, as $100$ is roughly the average number of round-trip times taken for a data transfer on the Internet~\citep{Corless2016}. The weight of each edge was initialized to the maximum capacity, $C$, in order to immediately experience contention. \todo{The mean path length of the artificial network was 4.7; this is large enough to provide several links of potential contention.} Each performance measure was averaged over 25 repeat simulations.\\
\ \\ 
\noindent \textbf{Parameter variation.} For the weight-update rules, we varied: AI ($k_i \in [1.0,9.0]$), SD ($k_d \in [1.0,9.0]$), MI $(k_i \in [1.1,1.9]$), MD $(k_d \in [0.1,0.9]$).\\
\ \\
\noindent \textbf{Real-world data.} We used the CAIDA Autonomous System (AS) relationship data \todo{to generate a graph based on the Internet routing network}~\citep{Caida}. Each AS represents the highest-level routing subnetwork on the Internet. The CAIDA data contained connectivity between 53,195 AS subnetworks. We treat each AS as a single routing node in our model. We created the same number of sources and targets (53,195 each); each source was connected to one random AS and each target had one incoming edge from a random AS. The network contained 537,582 total directed edges. To simulate flows, we selected $\sim1\%$ of the sources (500) and paired each source with a random target. Each performance measure was averaged over 10 repeat simulations. \todo{As before, we set the capacity, $C=1000$.}

\section{Results}

First, we describe the performance of each weight-update algorithm against the global objectives (bandwidth, drops, and queueing) using both simulated and real-world networks. Second, we support these results by analytically deriving the performance of each algorithm as it adapts to changing traffic demands. Third, we describe how the best performing rules compare to those commonly used in engineering.

\subsection{Observations from simulations and real-world network flows}
\label{sec:simresult}


We first compared the performance of the seven edge-weight update algorithms (Section~\ref{sec:alg}) via simulation. Each algorithm was evaluated according to the bandwidth offered and the amount of data that was dropped or queued. 



The two strongest performing algorithms were AIMD and MIMD, with Oja lying in between (Figure~2). AIMD, for some parameters, achieved a bandwidth comparable to other algorithms, but its main strength was in reducing the drop penalty by at least one order of magnitude (averages: $\textrm{AIMD}=6.4\%$, $\textrm{OJA}=69.2\%$, $\textrm{MIMD}=69.3\%$, $\textrm{AISD}=124.4\%$, $\textrm{MISD}=145.2\%$; Figure~2A). The former as due to its conservative rule for increasing edge weights for LTP (additive), and the latter was due to its aggressive edge weight decrease for LTD (multiplicative). AISD and MISD showed very high drop penalty primarily because upon contention, edge weights were only decreased subtractively; this led to slow adaptation, though higher bandwidth because few links were ever under-utilized. In general, the Oja rule (a variant of AISD) improved over AISD but still achieved a much higher drop penalty compared to AIMD also due to its conservative decrease. MIMD showed great sensitivity to algorithm parameters, some of which performed well. While keeping all edge weights at maximum capacity may be intuitively appealing (Max Send), this is in general not a good solution because any downstream bottleneck will result in massive data drops. We also observed similar trends for all algorithms using the queue model (Figure~2B). 

Overall, for both models, the multiplicative decrease algorithms (AIMD, MIMD) and the Oja algorithm demonstrated a better trade-off between bandwidth and drop/queue penalties than other algorithms, indicating their ability to more closely approach the optimal edge weights. We also tested these observations for larger networks with different router connectivity topologies and observed similar overall trends (Appendix, Figure~S1).



Next, we tested how well and how quickly each algorithm could adapt to new traffic, simulated on a real Internet backbone routing network. The simulation was for $3000$ time steps: During $t < 1000$ and $t > 2000$, 500 flows concurrently competed. During $1000 \leq t \leq 2000$, 500 additional flows were temporarily added (``rush hour''). All algorithms demonstrated some reduction in bandwidth when rush hour begins due to the additional number of competing flows. \todo{However, AIMD and MIMD incurred the least transient drop penalty (i.e.,\@ the drop penalty incurred immediately after rush hour starts; Figure~3B). Overall, AIMD yielding significantly lower transient drop penalty than all other algorithms ($P < 0.01$; 2-sample t-tests), and significantly lower overall drop penalty than MIMD.} AIMD and MIMD also produced only a 4--5\% difference in bandwidth compared to the other algorithms (average data per timestep: \todo{$\textrm{AIMD}=776 \pm 6.82$, $\textrm{MIMD}=786 \pm 6.94$, $\textrm{OJA}=814 \pm 19.38$, $\textrm{AISD}=819 \pm 6.11$, $\textrm{MISD}=820 \pm 5.93$ at $t=1500$, Figure~3A).} These results suggest that AIMD and MIMD adapt faster to changing traffic, yielding fewer transient drops with comparable bandwidth.


Next, to support these observations, we formally analyze the adaptive behavior of each algorithm.

\subsection{Analyzing transient response times for AIMD, MIMD, AISD, MISD, and Oja}

An important aspect of algorithm performance is its non-stationary or transient behavior; i.e.\@ how well it adapts when traffic suddenly increases and the available bandwidth per flow decreases (Figure~3). Real networks are never static; all flows experience some level of perturbation due to varying traffic. Thus, as opposed to analyzing convergence properties (as is typically done for gradient descent algorithms), we analyzed a simple but informative scenario: the performance of each algorithm when a second flow is added to a link that is initially serving only a single flow at maximum capacity. 

Assume the link $(u,v)$ under consideration has fixed weight $C$, and that there are two flows starting from $s_1$ and $s_2$ that both need to use $(u,v)$ to reach their downstream target (and no other link in the network is limiting). Let $W_{s_{1}u}(t)=C$ and then when the second flow begins, assume $W_{s_{2}u}(t)=1$, where $C \gg 1$. Assume a single LTD operation on $(u,v)$ will lower the total traffic sent by both flows along $(u,v)$ to be below $C$, hence alleviating the congestion. After the initial congestion event there will be $\geq 1$ time steps of LTP before congestion re-occurs. Let $n$ be the number of time steps before congestion re-occurs. We wish to find the amount of data dropped/queued (called \emph{overshoot}) at time $n$, defined as the difference between the amount of data desired by both flows along edge $(u,v)$ and $C$ at the moment LTD is re-activated. 
\begin{thm}[Overshoot of AIMD]\label{thm:AIMDos} The two-flow transient response of AIMD has an overshoot that increases linearly with $n$, but is \todo{reduced by a term proportional to the capacity, $C$.}
\end{thm}
\begin{proof}
At the first time step ($-1$ to simplify notation), the amount desired by both flows on $(u,v)$ is:
\begin{center}
\begin{tabular}{cccc}
Time $t$ & Flow 1 & Flow 2  & $\mathrm{Jam}_{uv}(t)$ \\ \midrule
$-1$ & $C$ & $1$ & True \\
\end{tabular}
\end{center}
Since the sum of the flows is greater than $C$, congestion occurs and the jam indicator variable $\mathrm{Jam}_{uv}(t)$ is true. In the next step, LTD via multiplicative decrease is applied:
\begin{center}
\begin{tabular}{cccc}
Time $t$& Flow 1 & Flow 2  & $\mathrm{Jam}_{uv}(t)$\\ \midrule
$0$ & $Ck_d$ & $k_d$ & False    \\
\end{tabular}
\end{center}
This brings the total desired traffic along $(u,v)$ under $C$. Additive increase then occurs for $n$ steps: 
\begin{center}
\begin{tabular}{cllc}
Time $t$& Flow 1 & Flow 2  & $\mathrm{Jam}_{uv}(t)$ \\ \midrule
$1$ & $C k_d + k_i$ & $k_d+k_i$ & False \\
$2$ & $C k_d + 2k_i$ & $k_d+2k_i$ & False\\
\vdots & \vdots & \vdots & \vdots \\
$n$ & $C k_d + n k_i$ & $k_d+nk_i$ & True \\
\end{tabular}
\end{center}
At time step $n$, the overshoot along $(u,v)$, which is the excess traffic over $C$, is:
\begin{align}
\label{eqn:aimdexactOS}
(\mathrm{Flow\ 1} + \mathrm{Flow\ 2}) - C & = (C k_d+n k_i + k_d + n k_i) - C, \notag \\
& \approx 2 nk_i + C\left(k_d-1\right),
\end{align}
assuming the single $k_d$ term is negligibly small. Note that since $k_d\in (0,1)$, the second term is negative.
\end{proof}
%
%
\begin{table}[htb] 
\begin{center}
\begin{footnotesize}
\caption{\textbf{Analysis of overshoot of weight-update algorithms.} Analytic (left) and simulated (right) overshoot values of each algorithm using $C=1000$. The Oja-based algorithm always under-shoots for the simple two-flow case; however, a closed form solution appears difficult to derive because the algorithm is dependent on edge traffic. \todo{The optimal solution (OPT) has zero overshoot.} For simulations of other algorithms, we selected 3 sets of parameters: \emph{\underline{Balanced}}: additive ($k_i=1.0, k_d=5.0$), multiplicative ($k_i=1.1,k_d=0.5$). \emph{\underline{Aggressive Increase}}: additive ($k_i=100.0, k_d=5.0$), multiplicative ($k_i=1.5,k_d=0.5$). \emph{\underline{Aggressive Decrease}}: additive ($k_i=1.0, k_d=100.0$), multiplicative ($k_i=1.1,k_d=0.1$). \todo{The variable $n$ corresponds to the number of time steps before congestion re-occurs.}}
\begin{tabular}{lc|ccc}\label{tbl:overshootsummary} \\ \toprule 
Rule & Transient Overshoot & Balanced ($n$) & Aggressive Increase ($n$) & Aggressive Decrease ($n$) \\ \midrule
AIMD & \(\displaystyle 2nk_i + C \left(k_d -1 \right) \) & 0.5 (250) & 101 (3) & 0.1 (450)\\
AISD & \(\displaystyle 2\left(nk_i-k_d\right) \) & 1.0 (5) & 196 (1) & 1.0 (100)\\
MIMD & \(\displaystyle C\left(k_d k_i^n - 1\right) \) & 72.8 (8) & 126 (2) & 84.6 (25)\\
MISD & \(\displaystyle C\left(k_i^n-1 \right) \) & 95.6 (1) & 494 (1) & 90.2 (2)\\ 
OJA  & Undershoots & -1 ($\infty$) & -1 ($\infty$) & -1 ($\infty$)\\
\todo{OPT}  & \todo{0 ($\infty$)}           & \todo{0 ($\infty$)}             & \todo{0 ($\infty$)}             & \todo{0 ($\infty$)}\\
\bottomrule
\end{tabular}
\end{footnotesize}
\end{center}
\end{table}

We similarly derived the overshoot of MIMD, AISD, and MISD (Appendix; summary in Table~\ref{tbl:overshootsummary}). \todo{The theoretical performance of the algorithms on the two-flow case correlate well with the simulated performance using hundreds of concurrent flows in a larger network.} Both AI algorithms have an overshoot that has a $2nk_i$ term. For AIMD, this term is reduced by $C(k_d-1)$. AISD, on the other hand, is hurt by a slow subtractive decrease, and hence only reduces the $2nk_i$ term by $2k_d$. Since $|C(k_d-1)| > |2k_d|$ for large values of $C$, AIMD typically performs better (compare blue dots vs.\@ yellow dots in Figure~\ref{fig:2}). 

The MIMD overshoot shows a complex dependence between $k_i$, $k_d$ and $C$, which makes performance highly  parameter-dependent, as we also observe via simulation (see variability of red dots in Figure~2). MISD has uniformly high overshoot, since $k_i^n > 1$ is multiplied by a large constant $C$, and thus performs poorly. Finally, Oja uses an AISD rule with a weight-dependent squared decay; this decay cancels out the additive increase term as the traffic of Flow 1 and 2 approaches $C$. This leads to performance that always under-shoots for this simple two-flow case, improving drop and queue penalty over AISD; however, in practice when many flows concurrently compete and overshoot does occur (as in Figures~2 and 3), it is limited by the weak decrease term, as the decay only doubles $k_d$, at best.

We also simulated the transient overshoot under the assumptions of Theorems 1--4 using different parameter settings (Table~\ref{tbl:overshootsummary}, right side). These simulations further validate the theorems, showing that AIMD overshoots the least, while MIMD varies from the second-best to second-worst, depending on $k_i$ and $k_d$. Further analysis of each algorithm is provided in the Appendix.


\subsection{Comparing distributed gradient descent algorithms in the brain and the Internet}


Simulations and theory both suggest that AIMD achieves a robust and well-balanced trade-off between bandwidth and drop/queue penalties, with MIMD and Oja also performing comparably depending on the parameters selected. This implies that these algorithms approach the optimal global edge-weight distribution more quickly and accurately than other distributed gradient descent algorithms, including MISD, AISD, Bang-Bang, and Max-Send. 

In the brain, the additive increase and multiplicative decrease rule (AIMD) for LTP and LTD, respectively, has strong theoretical and experimental support (Table~\ref{tbl:rw}), particularly over MI and SD models. The AIMD algorithm is also very similar to the rule proposed by~\citet{vanRossum2000} and has been commonly referred to as the mixed-weight update rule~\citep{Delgado2010}. This rule, in neural network simulations, has been shown to produce stable Hebbian learning compared to many other rules~\citep{vanRossum2000,Billings2009}. 

While Oja was one type of weight-dependent rule, other rules have also been proposed where a weak (i.e.\@ low-weight) synapse that undergoes LTP is strengthened by a greater amount than a strong synapse that undergoes LTP, and vice-versa for LTD (Table~\ref{tbl:rw}). Prior work has also suggested that individual synapses may have a ``memory'' that allows for more sophisticated update rules to be implemented, including history-dependent updates~\citep{Subhaneil2013}. \todo{One such form of update is called integral control in engineering, which utilizes the time integral of a variable from $t=0$ to the present}. We show the general forms of these rules in Table~\ref{tbl:rw} but do not explore them further here.


Interestingly, in engineering, AIMD also lies at the heart of the most popular congestion control algorithm used on the Internet today: the transmission control protocol (TCP~\citep{Corless2016}). In contrast to our \emph{per-link} model, congestion control on the Internet is performed on a global end-to-end basis \emph{per-flow}, meaning that the source of each flow regulates its transmission rate based on a binary feedback signal (called an ACK or acknowledgment) sent by the target. If there is a long delay before the ACK is received (or if the ACK is never received at all), congestion is assumed, and the source decreases its rate of sending data by a multiplicative constant (often $0.5$). Otherwise, the source increases its rate by an additive constant (often $1.0$). TCP was also designed with the goal of converging to steady-state flow rates over time in a gradient-descent like manner~\citep{Shakkottai2007,Jose2015}. Thus, despite different models and objectives, both the brain and the Internet may have discovered a similar distributed algorithm for optimizing network activity using sparse, local feedback.


\section{Discussion}

Our work connects distributed traffic flow control algorithms in engineered networks to synaptic plasticity rules used to regulate activity in neural circuits. \todo{While we do not claim that there is a one-to-one mapping between mechanisms of synaptic plasticity and flow control problems, we showed how both problems can be viewed abstractly in terms of gradient descent dynamics on global objectives, with simple local feedback.} We performed simulations and theoretical analyses of several edge-weight update rules and found that the additive-increase multiplicative-decrease (AIMD) algorithm performed the best in terms minimizing drops/queueing with comparable bandwidth as other algorithms. This algorithm also matched experimental data detailing the functional forms of LTP and LTD in the brain and on the Internet, suggesting a similar design principle used in biology and engineering. Further, these weight rules use limited (1-bit) local communication and hence may be useful for implementing energy-efficient and scalable flow control in other applications, including integrated circuits, wireless networks, or neuromorphic computing. 

There are many avenues for future work. First, other plasticity rules may also be explored within our framework, such as short-term plasticity. \todo{Second, in cases where source and/or receiver rates are fixed, the payload needs to be routed over alternative paths (i.e. routes may change over time basic on traffic~\citep{Isa2015}). This requires  that heavily used edges become down-weighted and unused edges become more attractive, which effectively performs load balancing over all resources (edges) in the network. Biologically, similar behavior is observed due to homeostatic plasticity mechanisms, which may inspire algorithms for this problem.} Third, these distributed gradient descent updates rules may be useful in machine learning applications for non-stationary learning. Fourth, more sophisticated weight-and history-dependent update rules already explored in engineering may provide insight into their form and function in the brain. Overall, we hope our work inspires closer collaborations between distributed computing theorists and neuroscientists~\citep{Navlakha2015}.

\clearpage
\section*{Figures}

\begin{figure}[h]
\begin{center}
\includegraphics[width=\textwidth]{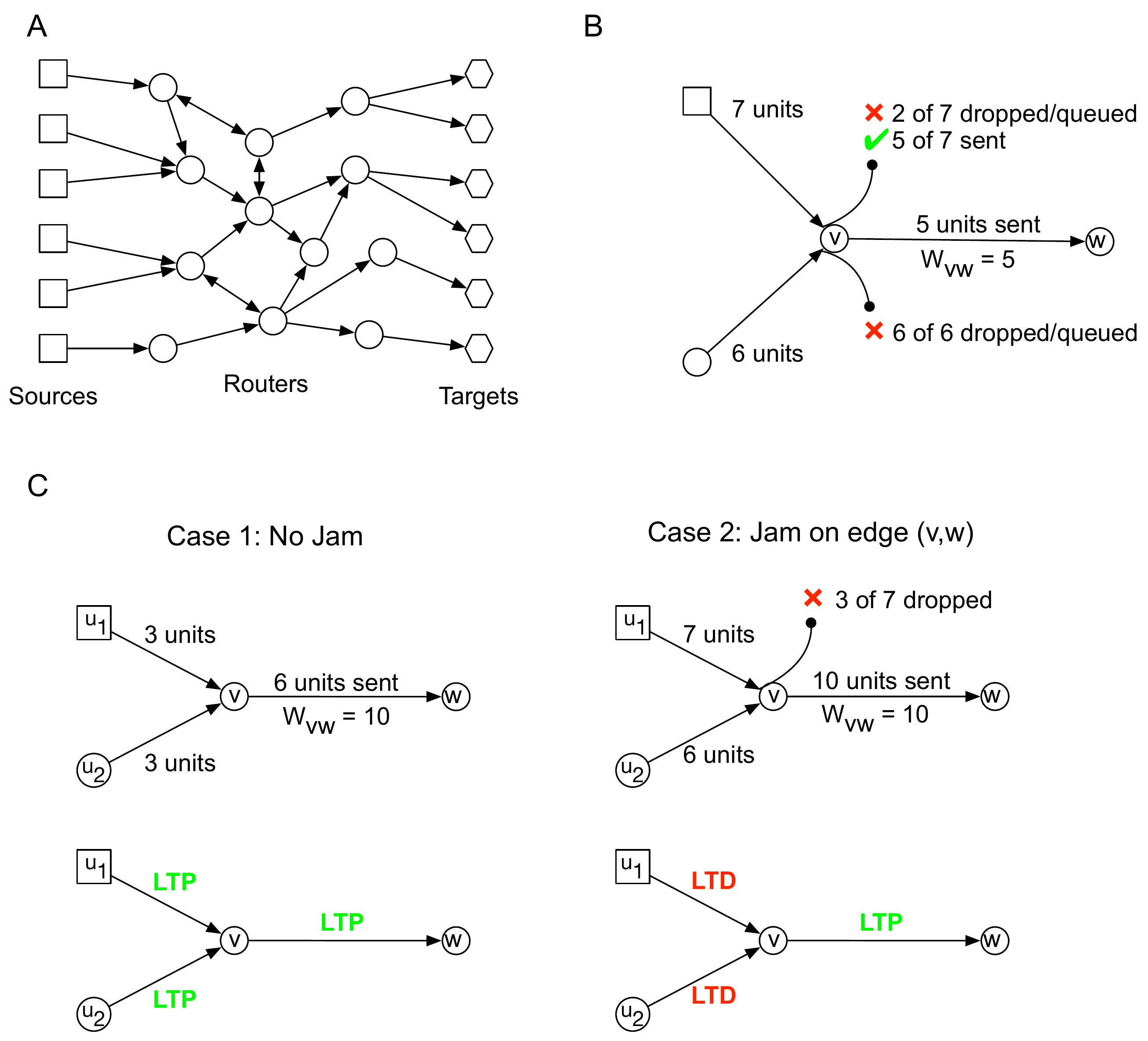}
\end{center}
\caption{\textbf{Model overview.} A) Input network, consisting of sources that transmit data to targets via a routing network. B) Illustration of a congested link $(v,w)$, where the amount of incoming data to a node exceeds outgoing link weight, leading to dropped or queued data. Flows incoming to node $v$ are processed in a random order. \todo{C) If no jam occurs (case 1), all links LTP in the next time step. If a jam does occur (case 2), then the edges contributing data to the jam undergo LTD, and the jammed link LTPs, in an attempt to alleviate the jam in the next time step.}}
\label{fig:1}
\end{figure}

\clearpage

\begin{figure}
\begin{center}
\includegraphics[width=\textwidth]{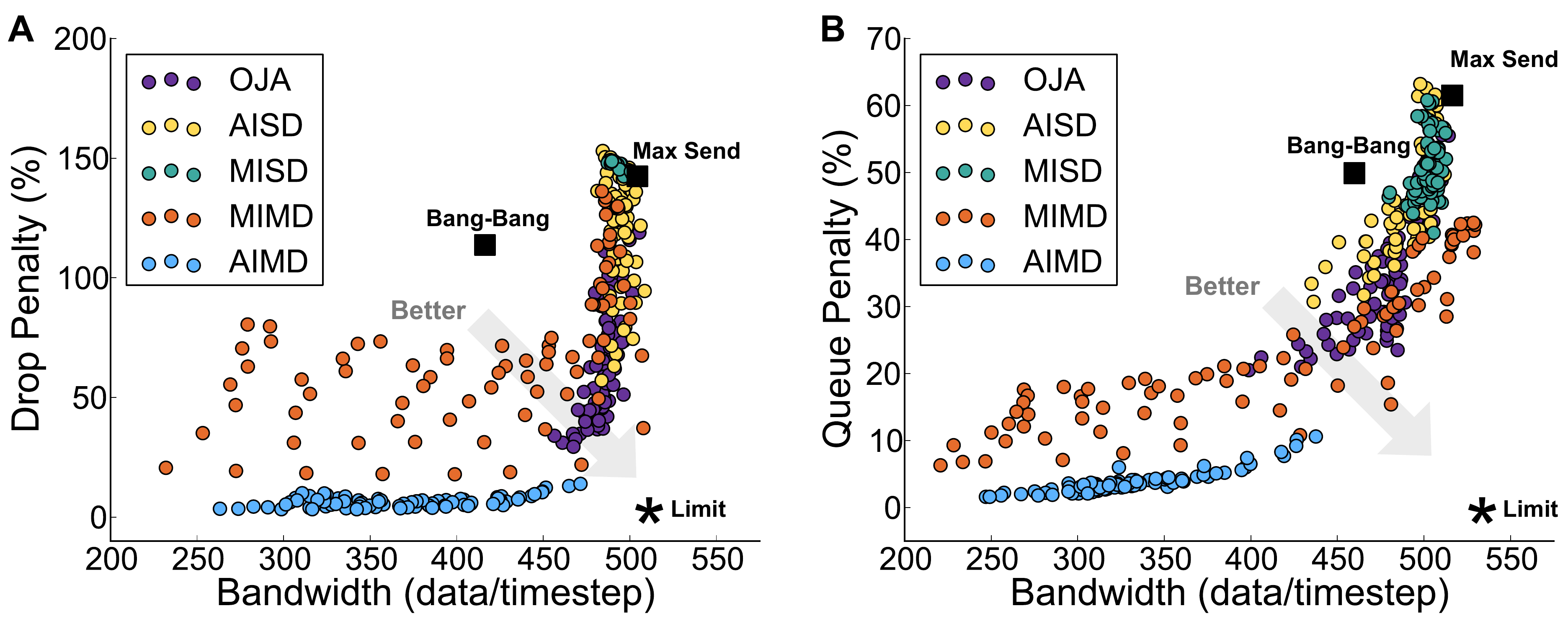}
\end{center}
\caption{\textbf{Comparison of seven activity-dependent weight update rules.} Bandwidth vs A) Drop penalty and B) Queue penalty. \todo{The lower right of the plot corresponds to an empirical upper bound, which occurs when bandwidth is that of the highest observed algorithm but with zero drop/queue penalty.} \todo{Each dot corresponds to an algorithm run using different values of the increase and decrease parameters $(k_i, k_d)$.} AIMD, MIMD, and Oja perform the best. 
\label{fig:2}}
\end{figure}

\clearpage

\begin{figure}
\begin{center}
\includegraphics[width=\textwidth]{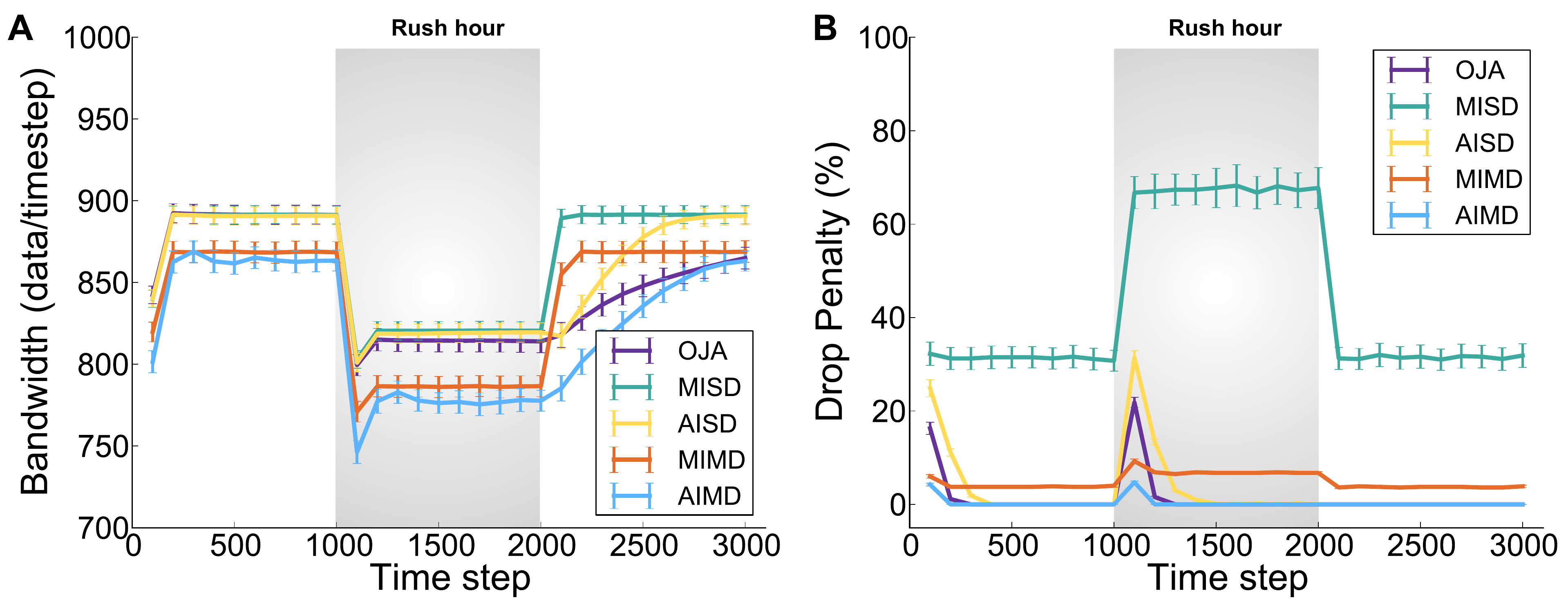}
\end{center}
\caption{\textbf{Adaptation to changing traffic demands.} Traffic simulated on a real Internet backbone routing network for $3000$ time steps. Performance measures were averaged every $100$ time steps. We selected the parameters for each algorithm that had the highest bandwidth in Figure~2 while being within 1\% of the minimum drop penalty for the algorithm. AIMD and MIMD showed the least additional penalty (B) due to rush hour, suggesting quick adaptation, while also yielding a similar bandwidth (A) as MISD, AISD, and Oja.
\label{fig:3}}
\end{figure}

\clearpage
\section*{Appendix}

\subsection*{Analyzing transient response times for AIMD, MIMD, AISD, MISD, and Oja}

Below, we derive the transient overshoot of AIMD, MIMD, AISD, and MISD.

\begin{thm}[Overshoot of AIMD] The two-flow transient response of AIMD has an overshoot that increases linearly with $n$, but is \todo{reduced by a term proportional to the capacity, $C$.}
\end{thm}
\begin{proof}
See the main text for the derivation. The overshoot is:
\begin{align}
(\mathrm{Flow\ 1} + \mathrm{Flow\ 2}) - C & = (C k_d+n k_i + k_d + n k_i) - C, \notag \\
& \approx C\left(k_d-1\right) + 2 nk_i,
\end{align}
assuming the single $k_d$ term is negligibly small.
\end{proof}

\begin{thm}[Overshoot of MIMD] The two-flow transient response of MIMD is highly parameter dependent: based on $k_d$ and $k_i$, the scaling of overshoot can either be dominated by $C$ or by a power of $n$.
\end{thm}
\begin{proof}
Following the proof of Theorem~1 in the main text, the time evolution will be:
\begin{center}
\begin{tabular}{cllc}
Time $t$ & Flow 1 & Flow 2  & $\mathrm{Jam}_{uv}(t)$ \\ \midrule
$-1$ & $C$ & $1$ & True \\
$0$ & $C k_d$ & $k_d$ & False \\
$1$ & $C k_d k_i$ & $k_d k_i$ & False \\
$2$ & $C k_d k_i^2$ & $k_d k_i^2$ & False \\
\vdots & \vdots & \vdots & \vdots \\
$n$ & $C k_d k_i^n$ & $k_d k_i^n$ & True \\
\end{tabular}
\end{center}
The overshoot for MIMD is:
\begin{align}
(\mathrm{Flow\ 1} + \mathrm{Flow\ 2}) - C & = k_d k_i^n (C+1)-C.\notag \\
& \approx C (k_d k_i^n - 1),
\end{align}
if $C \gg 1$. Thus, the overshoot shows a positive dependence on $C$ (unlike AIMD, which has a negative dependence on $C$) and a power dependence on $n$.
\end{proof}

\begin{thm}[Overshoot of AISD] The two-flow transient response of AISD increases linearly with $n$, but does not scale relative to the link capacity, $C$.
\end{thm}
\begin{proof}
The time evolution will be:
\begin{center}
\begin{tabular}{cllc}
Time $t$& Flow 1 & Flow 2  & $\mathrm{Jam}_{uv}(t)$ \\ \midrule
$-1$ & $C$ & $1$ & True \\
$0$  & $C- k_d$ & $1-k_d$ & False    \\
$1$  & $C- k_d+ k_i$ & $1-k_d+ k_i$ & False \\
$2$  & $C-k_d+2 k_i$ & $1-k_d+2 k_i$ & False \\
\vdots  &   \vdots  & \vdots & \vdots\\
$n$ & $C-k_d+n k_i$ & $1-k_d+n k_i$ & True \\
\end{tabular}
\end{center}
The overshoot for AISD is thus:
\begin{align} \label{aisdos}
(\mathrm{Flow\ 1} + \mathrm{Flow\ 2}) - C & = 2(nk_i-k_d).
\end{align}
\end{proof}

\todo{In our theorems, we ignore the constraint that $W\geq 1$. This limit in our implementation was due to integer rounding to ensure MISD links do not get stuck at $0$ weight; hence, the theorems we present here are more general. If we apply this limit, it affects the weight of flow 2 at $t=0$, which should equal $1$. In all cases except AISD, the difference is removed by the approximation at the end. For AISD, Eqn.\@~\eqref{aisdos} becomes $2nk_i-k_d$, a negligible change when $n\gg 1$.}

%
%
\begin{thm}[Overshoot of MISD] The two-flow transient response of MISD increases to the power of $n$ and as a product with $C$.
\end{thm}
\begin{proof}
The time evolution will be:
\begin{center}
\begin{tabular}{cllc}
Time $t$& Flow 1 & Flow 2  & $\mathrm{Jam}_{uv}(t)$ \\ \midrule
$-1$ & $C$ & $1$ & True \\
$0$ & $C- k_d$ & $1-k_d$ & False \\
$1$ & $(C-k_d)k_i$ & $(1-k_d) k_i$ & False \\
$2$ & $(C-k_d)k_i^2$ & $(1-k_d) k_i^2$ & False \\
\vdots  &   \vdots  & \vdots & \vdots \\
$n$ & $(C-k_d)k_i^n$ & $(1-k_d) k_i^n$ & True \\
\end{tabular}
\end{center}
The overshoot for MISD is thus:
\begin{align}
(\mathrm{Flow\ 1} + \mathrm{Flow\ 2}) - C & = k_i^n (C-2k_d+1) -C \notag \\
& \approx C\left( k_i^n -1 \right),
\end{align}
since $C$ is large relative to $2k_d+1$.
\end{proof}

General conclusions can be drawn from these relations that correlate well with the simulation results. We can bound the overshoot of both AI algorithms (AIMD and AISD) by $|\mathcal{F}| k_i$, where $|\mathcal{F}|$ is the number of flows that must share a link, because each individual flow will not overshoot more than one $k_i$. The more precise overshoot of AIMD, $C(k_d-1)+2nk_i$ (Theorem~1 main text), shows a significant capacity-dependent stabilizing factor, as $C\left(k_d-1\right)$, which is negative, counteracts the factor of $n$. When we repeated the AIMD flow simulation (Table~1 main text) with $C=50$ (as opposed to $C=1000$), the overshoot of AIMD increased, as expected.

AISD has a conservative increase similar to AIMD but suffers because subtractive decrease slowly adjusts when available bandwidth decreases. While some parameter settings may overcome this, it is difficult in practice to tune the SD constant to be effective in all scenarios. Simulation results supports this observation, showing a large drop penalty during rush hour and a slow recovery after traffic subsides (Figure~3 main text). For AISD, if multiple flows share a link, transient overshoot is approximately $|\mathcal{F}| \left( n k_i \right)$. When a large number of flows share a link, e.g.\@ millions of flows on an Internet backbone, overshoot will be very large.

The Oja-inspired algorithm is based on AISD, but subtracts a normalized traffic-dependent quadratic term. This attempts to correct for the slow LTD decrease of AISD, especially at high weights, and we do observe improvement over AISD in our network simulations. However, we also observe larger drop and queue penalties for Oja compared to AIMD due to its still rather conservative decrease following congestion. 
The decay term can also completely counteract the AI contribution when $D_{uv}\rightarrow \sqrt{WC}$ (i.e.\@ when edge utilization is high), thus potentially leading to edge under-utilization. This convergence to a weight less than $C$ causes LTD never to be triggered (hence, the $\infty$ term in Table 2 in the main text, for this simple two-flow case). The lack of periodic overshooting appears to be key for the drop and queue penalty improvements over AISD. 


The overshoot and performance of MIMD are highly dependent on $C$ relative to the $k_i$ and $k_d$ parameters, explaining the scattered performance of MIMD in Figure 2 (main text). With optimal parameter tuning, MIMD can be made to operate well under transient behavior, which is seen when certain MIMD points reach the AIMD region (Figure~3 main text). 

MISD can easily be seen as worse than AISD in terms of drop/queue penalties in our simulations (Figures~2 and 3 main text). Intuitively, this is because MISD reduces slowly (subtractively) but increases aggressively (multiplicatively). The transient overshoot analysis for MISD shows that it increases as a product of $C$ and a constant to the power of $n$. Since the decrease term is weak, $n$ will be small, meaning that poor performance will be especially seen for high capacity links.

\subsection*{Algorithm performance with additional topologies}

We performed simulations with 10-­fold larger networks $(n=1000)$, with both uniform and scale­-free degree distributions. \todo{The scale-free topology was derived using the Barabasi­-Albert preferential attachment model~\citep{Barabasi1999}. When generating graphs using this model, new nodes are connected to existing nodes with probability proportional to their existing degree. This mechanism has been shown to produce a power-law degree distribution.} We found no qualitative change in our conclusions here (Figure~\ref{fig:4} below) compared to the results discussed in the main text. Thus, our results are applicable to at least two classes of network topologies: uniform (inspired by grid-­like road networks) and power­-law (Internet). This invariance is likely due to the distributed nature of the flow control algorithms. The overshoot theorems exemplify this, having no assumptions on network size and topology.

\begin{figure}[!h]
\centerline{\includegraphics[width=\textwidth]{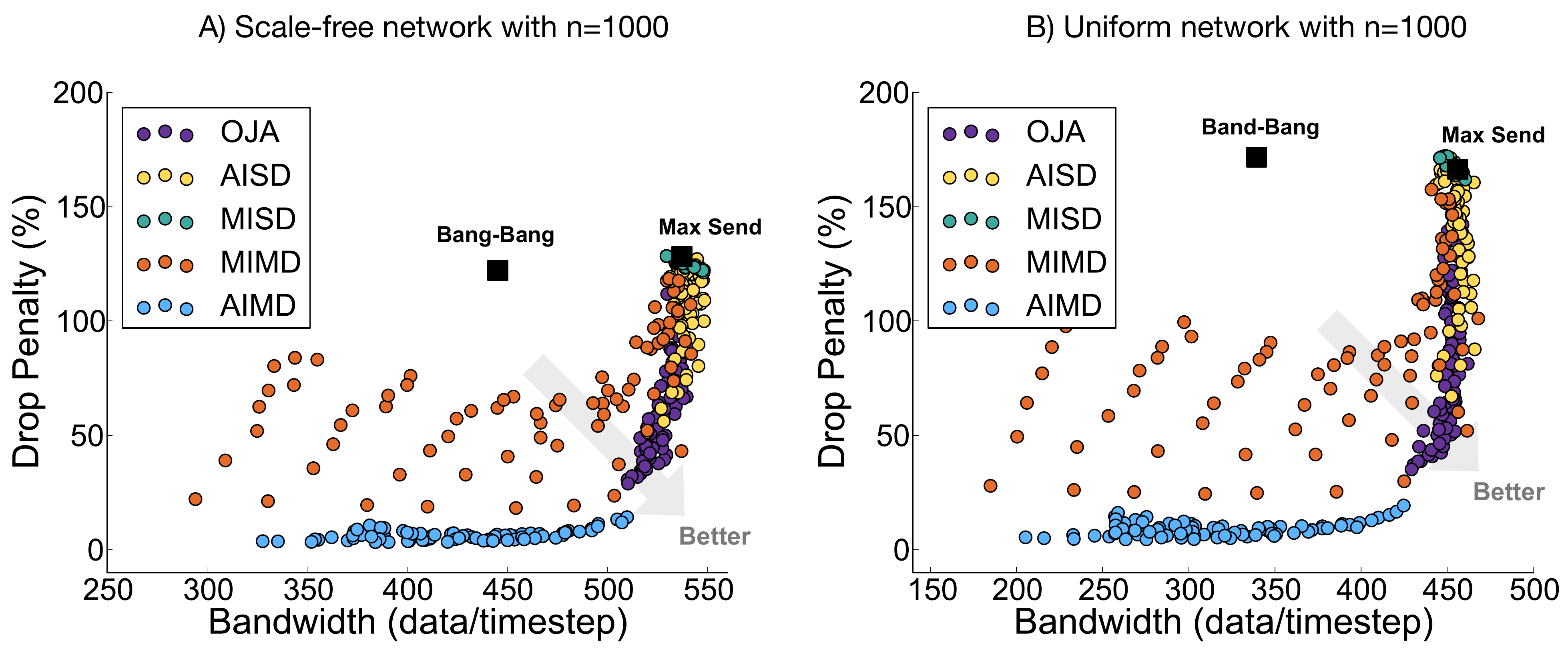}}
\caption{\textbf{Performance using additional topologies.} Comparison of all the algorithms using a scale-free degree distribution (A) and a uniform degree distribution (B) with 1000 nodes.
\label{fig:4}}
\end{figure}

\subsection*{\todo{Changes in edge weights (W) under dynamic traffic}}

\todo{To study how $W$ (edge weights) changed under the rush hour protocol of Figure~\ref{fig:3}, we plotted the average $W$ (computed in 10 time-step bins) for all used source-to-router links in the network. We focused on source-router links as these primarily control the amount of new data injected into the network in each time step. Error bars in this figure correspond to the standard error of the average over 10 trials. Both MIMD and AIMD reduce edge weights in response to rush-hour, showing that they handle excess traffic by reducing bandwidth instead of dropping data. The characteristic oscillatory probing nature of both algorithms is also apparent. In both cases, the range of W in each time step is narrowly bounded, indicating network stability.}

\begin{figure}[!h]
\centerline{\includegraphics[width=0.7\textwidth]{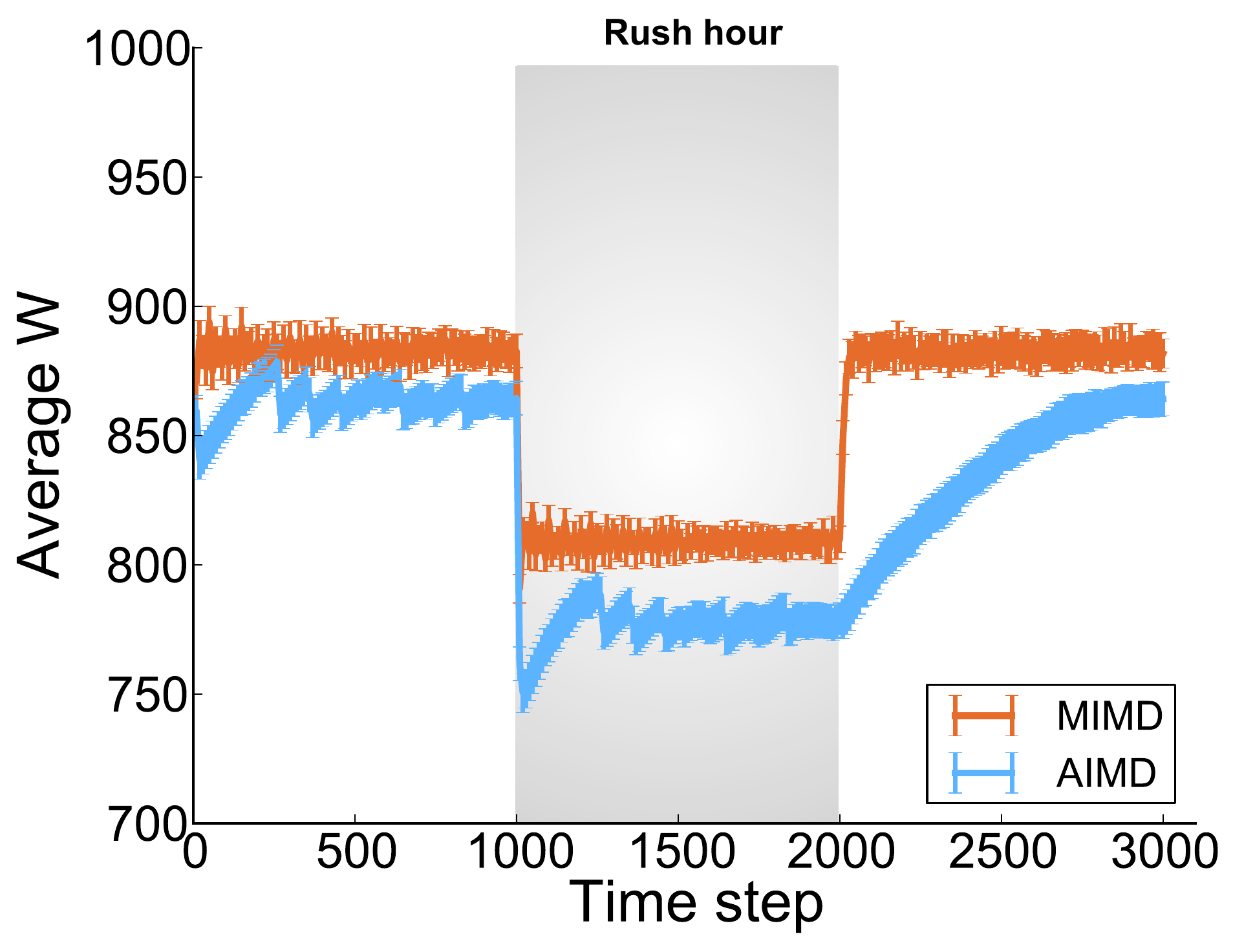}}
\caption{\todo{\textbf{Changes in the edge weight under the dynamic traffic protocol.} Error bars correspond to the standard error of the average over 10 trials.}
\label{fig:5}}
\end{figure}

\clearpage

\begin{small}
\bibliographystyle{apalike}
\bibliography{references}
\end{small}

\end{document}